\documentclass{article}

\usepackage{microtype}
\usepackage{graphicx}
\usepackage{subfigure}
\usepackage{booktabs} % for professional tables

\usepackage[hidelinks]{hyperref}

\usepackage[accepted]{icml2020} %TODO comment this out for submission

\usepackage{wrapfig}
\usepackage{graphicx}
\usepackage[normalem]{ulem}
\usepackage{amsmath,amssymb,amsthm, mathtools}
\usepackage{mathrsfs}
\usepackage[algoruled,lined,linesnumbered]{algorithm2e} 
\usepackage{amsfonts}
\usepackage{bm}
\usepackage{xcolor}
\usepackage[makeroom]{cancel}
\usepackage[switch]{lineno}
\usepackage{xargs}      % to use more than one optional parameter in a new commands
\usepackage[colorinlistoftodos,prependcaption,textsize=tiny]{todonotes}
\newcommandx{\unsure}[2][1=]{\todo[linecolor=red,backgroundcolor=red!25,bordercolor=red,#1]{#2}}
\newcommandx{\change}[2][1=]{\todo[linecolor=blue,backgroundcolor=blue!25,bordercolor=blue,#1]{#2}}
\newcommandx{\info}[2][1=]{\todo[linecolor=OliveGreen,backgroundcolor=OliveGreen!25,bordercolor=OliveGreen,#1]{#2}}
\newcommandx{\improvement}[2][1=]{\todo[linecolor=Plum,backgroundcolor=Plum!25,bordercolor=Plum,#1]{#2}}
\hypersetup{
    colorlinks=false,
}
\usepackage{bookmark}
\newtheorem{thm}{Theorem}

\newcommand{\Actions}{\mathcal{A}}

\newcommand{\Obs}{\mathcal{O}}

\newcommand{\be}{\bm{\eta}}
\newcommand{\bw}{\mathbf{w}}
\newcommand{\s}{\mathbf{s}} % macro for a state
\newcommand{\bth}{\bm{\theta}}

\DeclareBoldMathCommand \bphi {\bm{\phi}}
\DeclareBoldMathCommand \be {\bm{\eta}}
\DeclarePairedDelimiter\norm{\lVert}{\rVert}
\DeclareMathOperator*{\argmax}{arg\,max}

\makeatletter
\newcommand{\removelatexerror}{\let\@latex@error\@gobble}
\makeatother
\usepackage{xr-hyper}
\usepackage{multirow}
\icmltitlerunning{Planning with Expectation Models for Control}

\begin{document}

\twocolumn[
\icmltitle{Planning with Expectation Models for Control}



\icmlsetsymbol{equal}{*}

\begin{icmlauthorlist}
\icmlauthor{Katya Kudashkina}{guelph}
\icmlauthor{Yi Wan}{alberta,amii}
\icmlauthor{Abhishek Naik}{alberta,amii}
\icmlauthor{Richard S. Sutton}{alberta,amii}
\end{icmlauthorlist}

\icmlaffiliation{guelph}{Department of Engineering, University of Guelph, Guelph, Canada}
\icmlaffiliation{alberta}{Department of Computing Science, University of Alberta, Edmonton, Canada}
\icmlaffiliation{amii}{Alberta Machine Intelligence Institute, Edmonton, Canada}

\icmlcorrespondingauthor{Katya Kudashkina}{ekudashk@uoguelph.ca}

\icmlkeywords{model-based reinforcement learning, stochastic environment, control problem, non-Markov, expectation models}

\vskip 0.3in
]



\printAffiliationsAndNotice{}  % leave blank if no need to mention equal contribution
%TODO

\begin{abstract}
In model-based reinforcement learning (MBRL), Wan et al.\ (2019) showed conditions under which the environment model could produce the \emph{expectation} of the next feature vector rather than the full distribution, or a sample thereof, with no loss in planning performance. Such~\emph{expectation models} are of interest when the environment is stochastic and non-stationary, and the model is approximate, such as when it is learned using function approximation. In these cases a full distribution model may be impractical and a sample model may be either more expensive computationally or of high variance. Wan et al.\ considered only planning for prediction to evaluate a fixed policy.
In this paper, we treat the control case---planning to improve and find a good approximate policy. We prove that planning with an expectation model must update a~\emph{state}-value function, not an \emph{action}-value function as previously suggested (e.g., Sorg \& Singh, 2010). This opens the question of how planning influences action selections. We consider three strategies for this and present general MBRL algorithms for each. We identify the strengths and weaknesses of these algorithms in computational experiments. Our algorithms and experiments are the first to treat MBRL with expectation models in a general setting.
\end{abstract}

Methods that scale with computation are most likely to stand the test of time. We refer to scaling with computation in the context of artificial intelligence (AI) as using more computation to provide a better \textit{approximate answer}. This is in contrast to with the common notion of scaling in computer science as using more computation for solving a bigger problem~\emph{exactly}. The recent success of modern machine learning techniques, in particular that of deep learning, is primarily due to their ability to leverage the ever-increasing computational power (driven by Moore's Law), as well as their generality in terms of dependence on data rather than hand-crafted features or rule-based techniques. Key to building general-purpose AI systems would be methods that scale with computation (Sutton, 2019).

RL is on its way to fully embrace scaling with computation---it already embraces scaling with computation in many aspects, for example by leveraging search techniques (e.g., Monte-Carlo Tree Search, see Browne et al., 2012; Finnsson \& Bj{\"o}rnsson, 2008), 
and modern deep learning techniques such as artificial neural networks.
Methods that resort to approximating functions rather than learning them exactly have not been fully investigated. Extending the techniques used in the simpler \textit{tabular} regime to this \textit{function approximation} regime is an obvious first step, but some of the ideas that have served us well in the past might actually be impeding progress on the new problem of interest. For example, Sutton~\& Barto (2018) showed that when dealing with feature vectors rather than underlying states, the common Bellman error objective is not learnable with any amount of experiential data. Recently, Naik et al.~(2019) showed that discounting is incompatible with function approximation in the case of continuing control tasks. Understanding function approximation in RL is key to building general-purpose intelligent systems that can learn to solve many tasks of arbitrary complexity in the real world.

Many of the methods in RL focus on
approximating value functions for a given fixed policy---referred to as the~\emph{prediction problem}
(e.g., Sutton et al., 1988; Singh et al., 1995; Wan et al., 2019).
The more challenging problem of learning the best behavior 
is known as the \textit{control problem}---that is approximating optimal policies and optimal value functions.
In the control problem, an RL agent learns 
within one of two broad frameworks: model-free RL and model-based RL (MBRL). In model-free RL, the agent relies solely on its observations to make decisions
(Sutton \& Barto, 2018).
In model-based RL, the agent has a~\textit{model} of the world, which it uses in conjunction with its observations to~\textit{plan} its decisions.
The process of taking a model as input and producing or
improving a policy for interacting with a modeled environment
is referred to as~\textit{planning}.
Models and planning are helpful.
One advantage of models and planning is that they are useful when the agent faces unfamiliar or novel situations---when the agent may have to consider states and actions that it has not experienced or seen before.
Planning can help the agent evaluate possible actions by rolling out hypothetical scenarios according to the model and then computing their expected future outcomes 
(Doll et al., 2012; Ha \& Schmidhuber, 2018; Sutton \& Barto, 2018).
Planning with function approximation remains a challenge in reinforcement learning today (Shariff and Szepesv{\'a}ri, 2020).
Planning can be performed with various kinds of models: distribution, sample, and expectation. Wan et al.~(2019) considered planning with an expectation model for the prediction problem within the function approximation setting. 

In this paper we extend Wan et al.'s (2019) work on the prediction problem to the more challenging control problem, 
in the context of stochastic and non-stationary environments.
This will involve several important definitions.
We start off with discussing important choices in MBRL
(Section~\ref{sec:MBRL_choices}).
This is followed by fundamentals on planning with expectation models in the {general context of function approximation}
(Section~\ref{sec:foundations}).
We then show (in Section~\ref{sec:planning_w_action_values}) that planning with an expectation model for control in stochastic non-stationary environments must update a state-value function and not an action-value function as previously suggested (e.g., Sorg \& Singh, 2010). Finally, we consider three ways in which actions can be selected when planning with state-value functions, and identify their relative strengths and weaknesses (Sections \ref{sec:3_algorithms}~\&~\ref{sec:illustrations_3_methods}).

\section{Choices in Model-Based RL}
\label{sec:MBRL_choices}
Model-based methods are an important part of reinforcement learning’s claim to provide a full account of intelligence. An intelligent agent should be able to model its environment and use that model flexibly and efficiently to plan its behavior.
In MBRL, models add knowledge to the agent in a way that policies and value functions do not (van Hasselt et al.,
2019).
Typically, a model receives a state and an action as inputs, and computes
the next state and reward (Kuvayev \& Sutton, 1996; Sutton et al., 2008; Hester \& Stone, 2011).
This output is used in planning to further improve policies and value functions.
In this section we discuss three important choices one needs to make in MBRL.

\textbf{Learned models~vs.~Experience Replay.}~One choice to make is where planning improvements could come from:~\emph{learned models} or~\emph{experience replay} (ER) (Lin, 1992).
In replay-based methods the agent plans using experience stored in the agent's memory.
Some replay-based examples include deep Q-networks (DQN)
(Mnih et al.~2013; 2015) and its variations: 
double-DQN (van Hasselt et al., 2016); 
% that is a form of Double Q-learning (van Hasselt, 2010);
extension of DQN to include prioritized sweeping %\footnote{See Moore \& Atkeson, 1993; Peng \& Williams, 1993.} 
(Schaul et al.\ 2016);
deep deterministic policy gradient (Lillicrap et al., 2016); 
% that combines deterministic policy gradient algorithm (Silver et al., 2014) with insights from DQN;
rainbow DQN (Hessel et al., 2018).

MBRL methods in which 1) the model is parameterized by some learnable weights,
2) the agent learns the parameters of the model, and 
3) the agent then uses the model to plan an improved policy, 
are referred to as planning with~\emph{learned parametric models}. 
Learned parametric models are used in the Dyna architecture (Sutton, 1991),
in normalized advantage functions method % (NAF)
that incorporates the learned model into the Q-learning algorithm based
on imagination rollouts (Gu et al., 2016), and
in MuZero (Schrittwieser et al., 2019). The latter is a combination of a replay-based method and a learned model: the model is trained by using trajectories that are sampled from the replay buffer.

If the environment is non-stationary, then the transitions stored in the replay buffer might be stale and can slow down or even hinder the learning progress.
In this work, we focus on the learned parametric models.

\textbf{Types of Learned Models.}
Another important choice to make in MBRL is a choice of a model type.
A model enables an agent to predict what would happen if actions are executed from states prior to actually executing them and without necessarily being in those states. 
Given a state and action, the model can predict a sample, an expectation, or a distribution of outcomes,
which results in three model-type possibilities.

The first possibility is when a model computes a probability $p$ of the next state as a result of the action taken by the agent. We refer to such model as a~\emph{distribution model}.
Such models have been used typically with an assumption of
a particular kind of distribution such as Gaussian (e.g., Chua et al., 2018).
For example, Deisenroth \& Rasmussen (2011) proposed a model-based policy search method based on probabilistic inference for learning control %(PILCO)
where a distribution model is learned using Gaussian processes.
Learning a distribution can be challenging:
1) distributions are potentially large objects (Wan et al., 2019); and
2) distribution models may require an efficient way
of representing and computing them,
and both of the tasks can be challenging (Kurutach et al., 2018).

The second possibility is for the model to compute a sample of the next state
rather than computing the full distribution.
We refer to such model as a~\emph{sample model}.
The output of sample models is more compact than the output of distribution models and
thus, more computationally feasible.
In this sense sample models are similar to experience replay.
Sample models have been a good solution to when a deterministic environment appears stochastic or non-stationary because of function approximation in the model.
Feinberg et al.~(2018) used a sample model to improve value estimates in a method called model-based value estimation. % (MVE).
Another example is simulated policy learning (Kaiser et al., 2019) % SimPle
in which a variational autoencoder (Kingma~\& Welling, 2014) is used to model the stochasticity of the environment.
A disadvantage of a sample model is an additional branching factor in planning,
as multiple samples need to be drawn to gain a representative prediction.

The third possibility is to have a model that produces an expectation of the next state instead of the probability $p$ as in distribution models.
We refer to this kind of models as~\emph{expectation models}.
Expectation models have been an obvious choice for deterministic environments (Oh et al.,\ 2015; Leibfried et al.,\ 2017; Kurutach et al.,\ 2018). Wan et al.\ (2019), Sutton et al.\ (2008),
and Parr et al.\ (2008) all used linear expectation models.
Wan et al.\ (2019) showed that expectation models can also be used for planning in stochastic environments
without loss in planning performance when using a linear value function.

In this work, we choose to focus on expectation models parameterized by learnable weights. 
Expectation models in control settings with function approximation have been previously explored.
Sorg \& Singh (2010) proposed model-based planning with an expectation model and 
illustrated the use of linear-option expectation models compared to primitive-action linear
expectation models on the continuous rooms world domain.
Jafferjee (2020) evaluated the imperfection of Dyna-style algorithms in the context of function approximation and learned models.
Buckman et al.\ (2018) % (STEVE)
used an ensemble of expectation models and action-value functions.

\textbf{Episodic vs. Continuing problems}
Another choice is that of the problem setting. Sequential decision making problems are typically episodic or continuing. Control with function approximation in continuing problems is arguably the problem setting that matters for AI (Naik et al., 2019; Sutton \& Barto, 2018: Chapter 10), but studying episodic problems is a natural stepping-stone towards the problem of AI. Hence we first study planning for control with function approximation in the episodic setting, extending Wan et al.'s (2019) work on the prediction problem.
The episodic formulation involves termination of episodes, and thus requires an explicit indication of termination in our definitions, which are laid out in the following section.

\section{Foundations of Planning with Function Approximation}
\label{sec:foundations}
Consider a dynamic system
that takes as input at timestep $t$ an action $A_t$
from a discrete set of actions $\Actions$
and emits observations $O_t \in \Obs$.
This is the~\emph{environment}.
The sequence of actions and observations is referred to as~\emph{history} and denoted with $H$.
The history at timestep $t$ is: $H_t \doteq A_0, O_1, A_1, O_2, ..., A_{t-1}, O_t$. A random timestep at which the episode terminates is denoted by $T$.
The agent also receives a scalar numerical reward at every timestep that is a function of the observation: $R_t \doteq f( O_t) \in \mathbb{R}$.
An agent-environment interaction---an \emph{episode}---is then as follows: 
% S0,A0,R1,S1,A1,R2,S2,A2,R3,... from the book, page 48
% A0,O1,A1,O2,A2,O3,A3,O4,..., page 464 of RL book
$A_0, O_1, R_1, ..., A_T, O_\bot, R_T$ 
where $O_\bot$ is a special observation at the terminal state.
The history is what an agent interacting with the environment can know and observe.
The agent chooses actions to maximize the expected \emph{discounted return} from the start state. At timestep $t$, the expected return discounted by factor $\gamma$ is denoted by $G_t$:
\fontsize{9}{6}
\begin{align}
  G_t \doteq \displaystyle\sum\limits_{k = t}^{T - 1} \gamma^{k-t} R_{k+1}, \text{ where } 0 \leq \gamma < 1.  
\end{align}
\normalsize

\textbf{States.} 
An important notion in function approximation is that of the~\emph{agent state}.
The agent state is different from the environment's state.
The environment's state is all the things used by the environment to compute itself.
It may include latent states that are not visible to the agent.
The agent state is an approximation of possible environment states.
The agent state at time $t$ is a compact summary of the agent's history up to $t$ (including past actions, observations, and rewards) that is useful for predicting and controlling the future of the trajectory.
In the function approximation setting, the agent state is represented by a feature vector of size $d$ 
denoted by $\s$, $\s \in \mathbb{R}^{d}$, that is computed using a~\emph{state-update} function $u$ (Sutton \& Barto, 2018).
The state-update function uses the most recent observation and action along with the most recent agent state to recursively compute the new agent state: 
$\s_t \doteq u (\s_{t-1}, A_{t-1}, O_t)$.
The state-update function $u$ may be constructed based on prior knowledge (e.g., Albus, 1971; 1981) or learned (e.g., Littman \& Sutton, 2001).
The agent state $\s$ is an approximation of an environment state and may not be Markov,
that is Pr$(\s_{t+1} |\,\s_t)$ may not be equal to Pr$(\s_{t+1} |\,H_t)$. Unless otherwise specified, henceforth we refer to the agent state just by~\emph{state}.
 
\textbf{Value functions.}
In the tabular RL setting, we define~\textit{true} value functions w.r.t.\ environment states. For instance, for a policy $\pi$, the true value function $v_{\pi}$ is the expected return given the environment state (which is a Markov state).
But in the function approximation setting, the agent does not have access to the environment state, nor is the agent state necessarily Markov. Hence the true value functions $v_{\pi}, v_*, q_{\pi}, q_*$ are undefined w.r.t.\ agent states.
Instead, we have approximate value functions $\hat v_{\pi}, \hat {\pi}$ parameterized by learnable weights (e.g., $\hat v(\s, \bw), \bw \in \mathbb{R}^{d})$.

In this work we focus on linear value functions.
One may think that this choice is limiting in the expressivity of the value function.
We argue that the state-update function $u$ can encompass the necessary expressivity of the features. 
In many cases, we can choose to have a
linear value function by changing a feature representation with $u$ function.
For example, to achieve that, a choice for a state-update function $u$ can be
variational autoencoders (e.g., Kingma \& Welling, 2014; Rezende et al., 2014), 
long short-term memory networks (Hochreiter \& Schmidhuber, 1997), 
predictive state representation (Littman \& Sutton, 2001),
temporal difference networks (Sutton et al., 2005), 
or generalized value functions (White, 2015).
Choices of linear representations have a long history
in the field of reinforcement learning and form the basis of many methods with
strong theoretical guarantees (Sorg \& Singh, 2010).

\textbf{Value Iteration}.
A known approach to learning a policy is iterative updates
to a state-value or action-value function (Sutton, 1988).
Various solutions exist for updating value functions.
A classic planning algorithm is \emph{value iteration} (Bellman, 1957; Sutton,
1988; Watkins, 1989; Bertsekas, 2012) which performs sweeps of computed backup values through the state or state-action space.

\textbf{Models.}
In the tabular setting, a model takes an environment state and action as input, and produces a sample or a distribution of the next environment state and reward.
In the function approximation setting, while the agent state is not Markov w.r.t.\ the environment, it is Markov w.r.t.\ the model. This means we can define models that use agent state instead of the environment state, and operate in the same way. Such models can be learned from experience, in other words, models can be parameterized by weights that can be learned.

\textbf{Definition.} A \emph{distribution model} consists of
an~\emph{expected reward function} $r:
\mathbb{R}^d \times \Actions \rightarrow \mathbb{R}$, 
% Note: this is to ensure that the episodes are finite.
written $r(\s,a)$, 
and a~\emph{transition probability function}
$p: \mathbb{R}^d \times \mathbb{R}^d \times \Actions \rightarrow [0, 1]$, 
written $p(\s'\mid \s, a)$, where:
\fontsize{9}{6}
    \begin{align}
                  \int\displaylimits_{\s' \in \mathbb{R}^{d}} p(\s' \mid \s, a) d\s' \leq 1\quad
      \forall\ \s \in \mathbb{R}^d, a \in \Actions.
      \label{eq:transition_part_of_distr_model}
\end{align}
\normalsize
If $p$ integrates to less than 1, then the remaining probability is interpreted as the probability of termination.

\textbf{Definition.} \emph{Approximate value iteration (AVI) with a distribution model} consists of repeated applications of the 
the following update:
\fontsize{9}{6}
\begin{align} 
    \bw \leftarrow \bw + 
        \alpha\big( g(\s,\bw) - \hat v(\s,\bw)\big) \nabla_{\bw} \hat v(\s,\bw),
        \label{eq:AVI_distr_model_weights}
\end{align} 
\normalsize
where $g(\s, \bw)$, the \emph{update target}, is 
\fontsize{8.5}{6}
    \begin{align} 
      g(\s, \bw) \doteq \max_{a \in \Actions} 
                \bigg[r(\s,a) + 
                    \gamma \int\displaylimits_{\s' \in \mathbb{R}^{d}} p(\s' \mid \s, a) \hat v(\s',\bw) d\s'\bigg].
    \label{eq:target_state_value}
    \end{align}
\normalsize

\textbf{Definition.}
A~\emph{general episodic expectation model (GEEM)} consists of an expected reward function
$r: \mathbb{R} \times \Actions \rightarrow [-\infty , 0]$, an expected next-state function $\hat \s: \mathbb{R}^d \times \Actions \rightarrow \mathbb{R}^d$, and a termination probability function
$\beta : \mathbb{R}^d \times \Actions \rightarrow [0, 1]$.
$\beta(\s, a)$ is interpreted as the probability of terminating in one step from state $s$ if action $a$ is taken, and
$\hat \s(\s, a)$ is interpreted as the expected next state if termination does not occur. 

\textbf{Definition.} \emph{Expectation value iteration (EVI)} consists of repeated application of~(\ref{eq:AVI_distr_model_weights}) with  the update target
\fontsize{8.5}{6}
\begin{align}
    g(\s, \bw)  = \max_{a \in \Actions} \bigg
    [ r(\s,a) + \gamma  \big(1 - \beta (\s,a)\big) \hat v\big(\hat \s(\s,a), \bw\big) \bigg].
    \label{eq:compute_target}
\end{align}
\normalsize

In the special case in which the approximate value function is linear, $\hat v(\s ,\bw) = \bw^{\top}\s$, 
this update target can be written
\fontsize{8.5}{6}
\begin{align}
    g(\s, \bw)  = \max_{a \in \Actions} \bigg
    [ r(\s,a) + \gamma  
    \underbrace{\big(1 - \beta (\s,a)\big)  \hat \s (\s, a)}_{\bar \s(\s,a)} {}^\top \bw \big) \bigg], 
    \label{eq:EVI_special_case}
\end{align}
\normalsize
revealing that the $\beta$ and $\hat\s$ functions of the expectation model can be combined and learned as a single function $\bar \s$.
This motivates the next definition.

\textbf{Definition.} 
A \emph{zero-terminal expectation model (ZTEM)} consists of 
an expected reward function
$r: \mathbb{R} \times \Actions \rightarrow [-\infty , 0]$ and a zero-terminal expected next-state function $\bar \s: \mathbb{R}^d \times \Actions \rightarrow \mathbb{R}^d$, written $\bar \s (\s,a)$, 
which can be interpreted as an expectation of the next state given that the terminal state is treated as a zero vector.

\textbf{Definition.} A ZTEM and a GEEM are \emph{aligned} iff their expected reward functions are identical and the expected next state functions are related by
\begin{align}
    \bar \s (\s, a) = \big(1 - \beta(\s, a)\big) \hat \s(\s, a). 
    \label{eq:ZTEM=GEEM}
\end{align}

\textbf{Definition.} \emph{Linear expectation value iteration (LEVI)} is the combination of EVI and a linear state-value function.
It consists of repeated applications of
~(\ref{eq:AVI_distr_model_weights})
with the update target
\begin{align}
    g(\s, \bw)  = \max_{a \in \Actions} \big
    [ r(\s,a) + \gamma \bw^\top \bar{\s}(\s,a) \big],
    \label{eq:LEVI_update_target}
\end{align}

\begin{thm}
If a ZTEM and GEEM are aligned, 
then LEVI is equivalent to EVI with a linear state-value function.
\end{thm}
\begin{proof}
Follows immediately from the combination of~(\ref{eq:EVI_special_case}) and~(\ref{eq:ZTEM=GEEM}).
\end{proof}

\textbf{Definition.} A \emph{ZTEM and a distribution model are aligned} iff their expected reward functions are identical and 
their transition parts are related by
\fontsize{8.5}{6}
    \begin{align}
      \bar \s (\s,a) = \int\displaylimits_{\s' \in \mathbb{R}^{d}} \s' p(\s' \mid \s, a) d\s'\quad
      \forall\ \s \in \mathbb{R}^d, a \in \Actions,
      \label{eq:exp_model_transition}
    \end{align}
\normalsize

The theorem below extends Wan et al.'s (2019) result to our setting and definitions.
\begin{thm}
If a ZTEM and a distribution model are aligned, 
then LEVI is equivalent to AVI with a linear state-value function.
\end{thm}
\begin{proof}
Consider now the integral part in~(\ref{eq:target_state_value}) below:
\fontsize{8.5}{6}
\begin{align} 
        \int\displaylimits_{\s' \in \mathbb{R}^{d}} p(\s'| \s, a) \hat v (\s', \bw) d\s'
        & = \int\displaylimits_{\s' \in \mathbb{R}^{d}} p (\s' \mid \s, a ) \bw^\top \s' d\s' \\
        & = \int\displaylimits_{\s' \in \mathbb{R}^{d}} p (\s' \mid \s, a) \s'^\top \bw d\s' \nonumber\\
        & =\bigg( \int\displaylimits_{\s' \in \mathbb{R}^{d}} p (\s' \mid \s, a ) \s' d\s'\bigg )^\top \bw\\
        & = \mathbb{E}_{\be} \big[ \s_{t+1} \mid \s_t = \s, A_t = a\big] ^\top \bw \\
        & \doteq \bar \s(\s, a )^\top \bw
\end{align}
\end{proof}
\normalsize
This shows that~(\ref{eq:target_state_value}) and (\ref{eq:LEVI_update_target}) are equivalent.
The theorem allows us to use the benefits of expectation models such as lesser computation and smaller variance. 

\textbf{Planning and backup distribution.}
Planning can often be described as proceeding in a sequence of \emph{state-value backups}, each of which updates the value estimate of a single state (and perhaps others by generalization).
That state, called the~\emph{backup state}, can be selected in many different ways.
Often we can speak of the distribution of states at which backups are done.
A state distribution from which backup updates are performed during planning
has an effect on the value function convergence.
Here we refer to a distribution of state we are planning from as a~\emph{backup distribution}.
In general, we talk of backup control which is analogous to ``search control'' strategy such as prioritized
sweeping (see Sutton et al., 2008; Pan et al., 2018).
To update the backup state, one or more backups are made either from it or from its possible successor states, a backup target is formed, and finally the backup-state's state-value estimate is shifted toward the target.
In one-step backups, all the backup updates are from the same backup state;
they all predict one step into the future from it.
These are the shortest backups. In~\emph{multi-steps}, or \emph{iterated}, backups,
backups are made both from the backup state and the predicted states of those backups. 

\section{Planning with Expectation Models Can Not Be Done with Action-value Functions}
\label{sec:planning_w_action_values}
It is natural to extend AVI to an action-value form.
Here we demonstrate that
planning with expectation models with function approximation and linear value functions must update state-value functions and not \emph{action-value} functions.
An~\emph{update target} for an~\emph{approximate action-value iteration} (AAVI)
is defined as $g(\s, a, \bw)$:
\fontsize{8}{6}
    \begin{align} 
      g(\s, a, \bw) \doteq r(\s,a) + 
                    \gamma \int\displaylimits_{\s' \in \mathbb{R}^{d}} p(\s' \mid \s, a) 
                    \max_{a'} \hat q(\s',a', \bw) d\s'
    \label{eq:target_action_value}
    \end{align}
\normalsize
AAVI update is then: 
\fontsize{8.5}{6}
\begin{align} 
    \bw \leftarrow \bw + 
        \alpha\big( g(\s, a, \bw) - \hat q(\s, a,\bw)\big) \nabla_{\bw} 
                                                        \hat q(\s, a, \bw).
\end{align} 
\normalsize

In stochastic environments, 
planning with expectation models with function approximation and linear value functions must update state-value functions and not action-value functions.
To update an action-value function, we need the second part of the target in \eqref{eq:target_action_value}, that is as follows:
\fontsize{8.5}{6}
\begin{align} 
        \int\displaylimits_{\s' \in \mathbb{R}^{d}} p(\s'|\s, a) \max_{a'}\hat q(\s', a')\,  d\s' &\doteq \nonumber%\\
    \int\displaylimits_{\s' \in \mathbb{R}^{d}} p(\s'|\s, a)\max_{a'}({\bw}_{\! a'}^\top \s')\, d\s'\\
    \label{eq:aavi}
    &\not=\max_{a'}\int\displaylimits_{\s' \in \mathbb{R}^{d}} p(\s'|\s, a) {\bw}_{\! a'}^\top \s' d\s'.
\end{align}
\normalsize

This result shows that Sorg \& Singh's (2010) usage of
an expectation model for planning with linear \textit{action}-value function is invalid. For 
a linear expectation model matrix ${\bf F}_{\!a}\in\mathbb{R}^d \times \mathbb{R}^d$, the second part of the target they used is $\max_{a'} ({\bf F}_{\!a}\s)^\top {\bw}_{\!a'}$, which we showed is not equivalent to an action-value value iteration update with an aligned distribution model \eqref{eq:aavi}:
\fontsize{8.5}{6}
\begin{align}
    \int\displaylimits_{\s' \in \mathbb{R}^{d}} p(\s'|\s, a)\max_{a'}\hat q(\s', a')\, d\s'
    &\not=\max_{a'} ({\bf F}_{\!a}\s)^\top {\bw}_{\!a'}.
\end{align}
\normalsize
We demonstrate that planning with expectation models cannot update action-value functions on a counterexample.

\textbf{A counterexample}.
Consider a Markov decision process (MDP) with three transition states shown in Figure~\ref{fig:monster_mdp}.
From the start state $1$ the agent can perform one of the two actions: $A$ and $B$.
\begin{figure}[h]
\begin{center}
\includegraphics[scale=0.3]{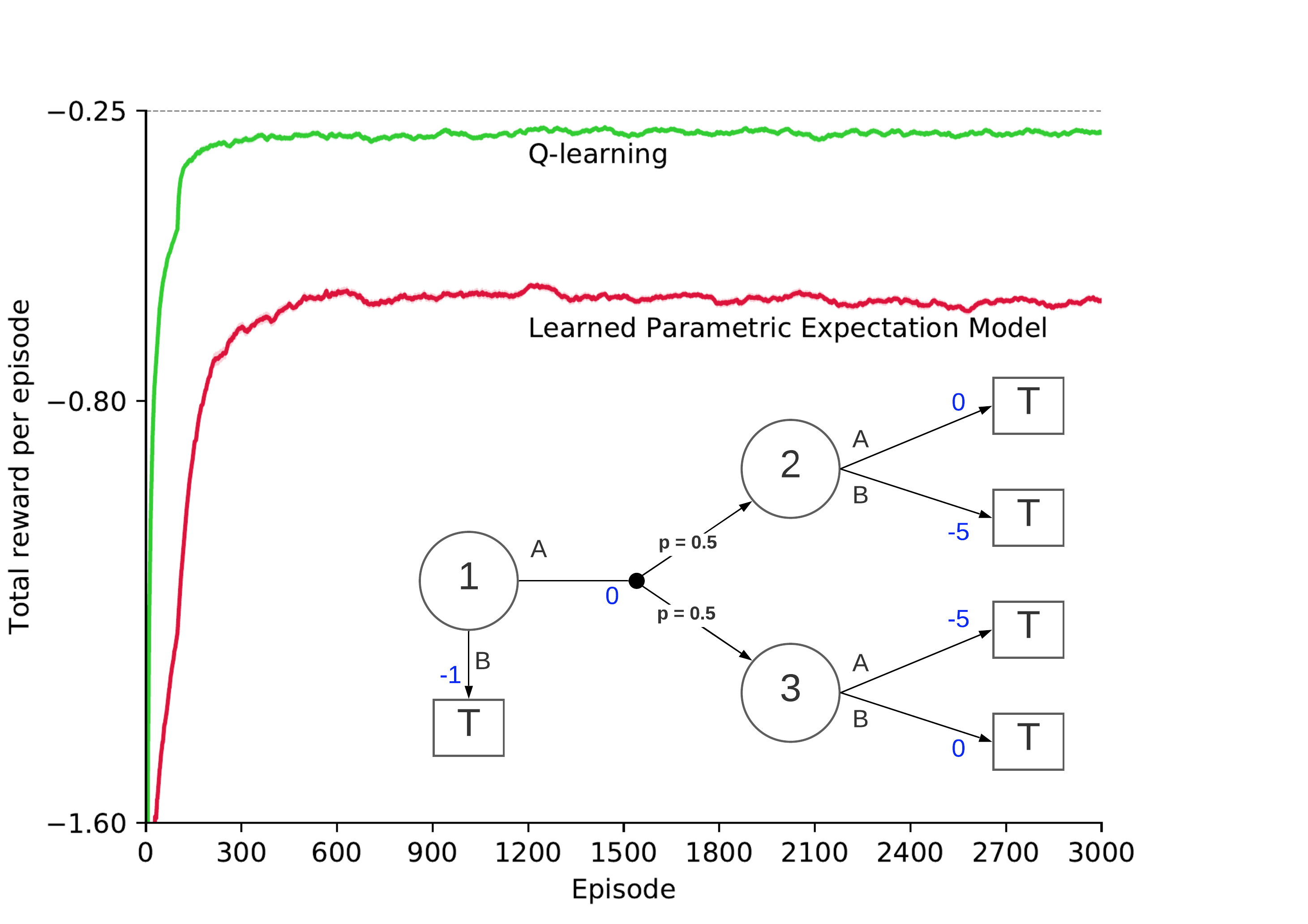}
\caption{{A counterexample episodic MDP that illustrates performance of planning with an expectation model and action values. 
}}
\label{fig:monster_mdp}
\end{center}
\end{figure}
With equal probability, the action $A$ causes the agent to advance either
to the states $2$ or $3$.
If the transition takes the agent to the $2$ state, then
action $A$ takes the agent to the terminal state with the reward $0$,
and action $B$ takes the agent to the terminal state with the reward $-5$.
The opposite happens in the $3$ state:
action $A$ takes the agent to the terminal state with the reward $-5$,
and action $B$
takes the agent to the terminal state with the reward $0$.
The episode ends once the agent reaches a terminal state.
Figure~\ref{fig:monster_mdp} empirically illustrates that expectation
models do not work with action values in
stochastic environments.
It shows one-step Q-learning (Watkins~\& Dayan, 1992) reaching optimal policy that is close to $-0.25$, while one-step Q-planning (Sutton~\& Barto, 2018: Chapter 8) with a learned expectation model and action values is not capable of that.

\textbf{A corridor domain illustration}.
We further extend the toy MDP to a more practical domain 
and illustrate
on a stochastic corridor that planning with expectation models cannot update action-value functions.

\begin{figure} [t!]
\vskip 0.2in 
\begin{center}
\includegraphics[scale=0.45]{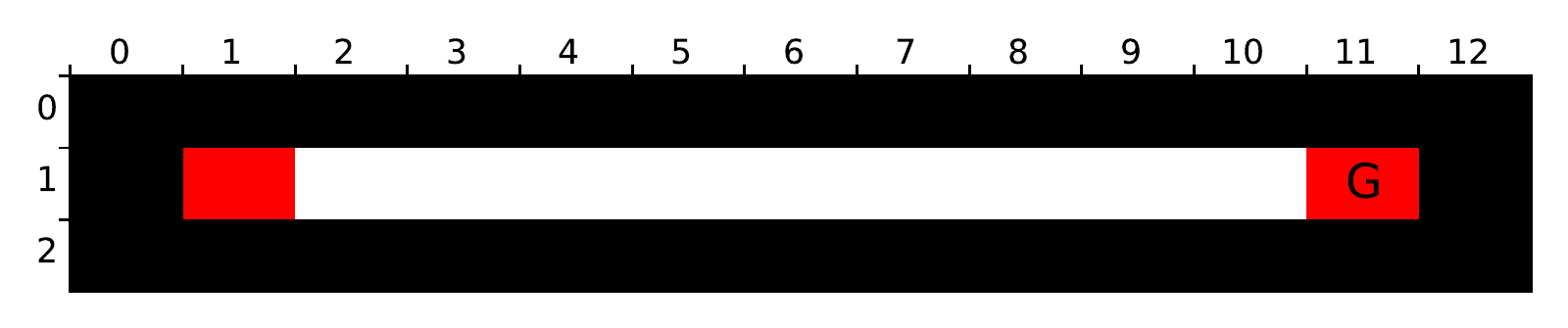}
\caption{A corridor domain.}
\label{fig:corridor}
\vskip -0.4in
\end{center}
\end{figure}

Consider a corridor domain shown on Figure~\ref{fig:corridor}.
The corridor consists of $9$ non-terminal states. 
The black cells represent walls, red cells represent terminal states. The agent can move in the white cells using one of the
\texttt{left} or \texttt{right} actions. 
The actions are stochastic: with probability 2/3, the actions cause the agent to move one cell in the direction corresponding to its name, and with probability 1/3, the agent ends up one cell in the other direction. 
An episode starts with the agent randomly initialized in one of the white states, and ends when it reaches one of two terminal states. The reward is $-1$ on all time steps.
If the agent reaches the terminal state marked by `G'---the goal state---it gets a reward of $+20$; the reward is $0$ for reaching the other terminal state.
In this illustration, we allowed a positive reward while ensuring that episodes terminate.

Figure~\ref{fig:long_run} demonstrates learning curves over $20000$ episodes.
A one-step Q-learning agent converged
to the performance of the agent planning with action-values and a true model, while planning with expectation model and action values could not reach optimal behavior.

\section{How Does Planning Affect Action Choice?}
\label{sec:3_algorithms}
The previous section showed that planning with expectation models
cannot proceed directly with action values.
In this section we discuss how the planning affects the choice of actions when expectation models are used with state-value functions.
The following algorithms use linear value functions with state
$\s$ that is a feature vector produced by a state-update function,
$r(\s,a)$  and $\bar \s$ that are computed by the model prediction.
They also use an expectation model whose parameters are learned by minimizing the squared loss as follows:
$\norm{\bar \s_{t+1} - \s_{t+1}}_2^2$ and 
$(\hat r_{t+1} - R_{t+1})^2$.
For planning updates, state $\s$ is sampled from a backup distribution.
All algorithms perform regular TD update outside of the planning updates\footnote{See supplementary materials for details.}.

\begin{figure} [t!]
\centering
\includegraphics[scale=0.5]{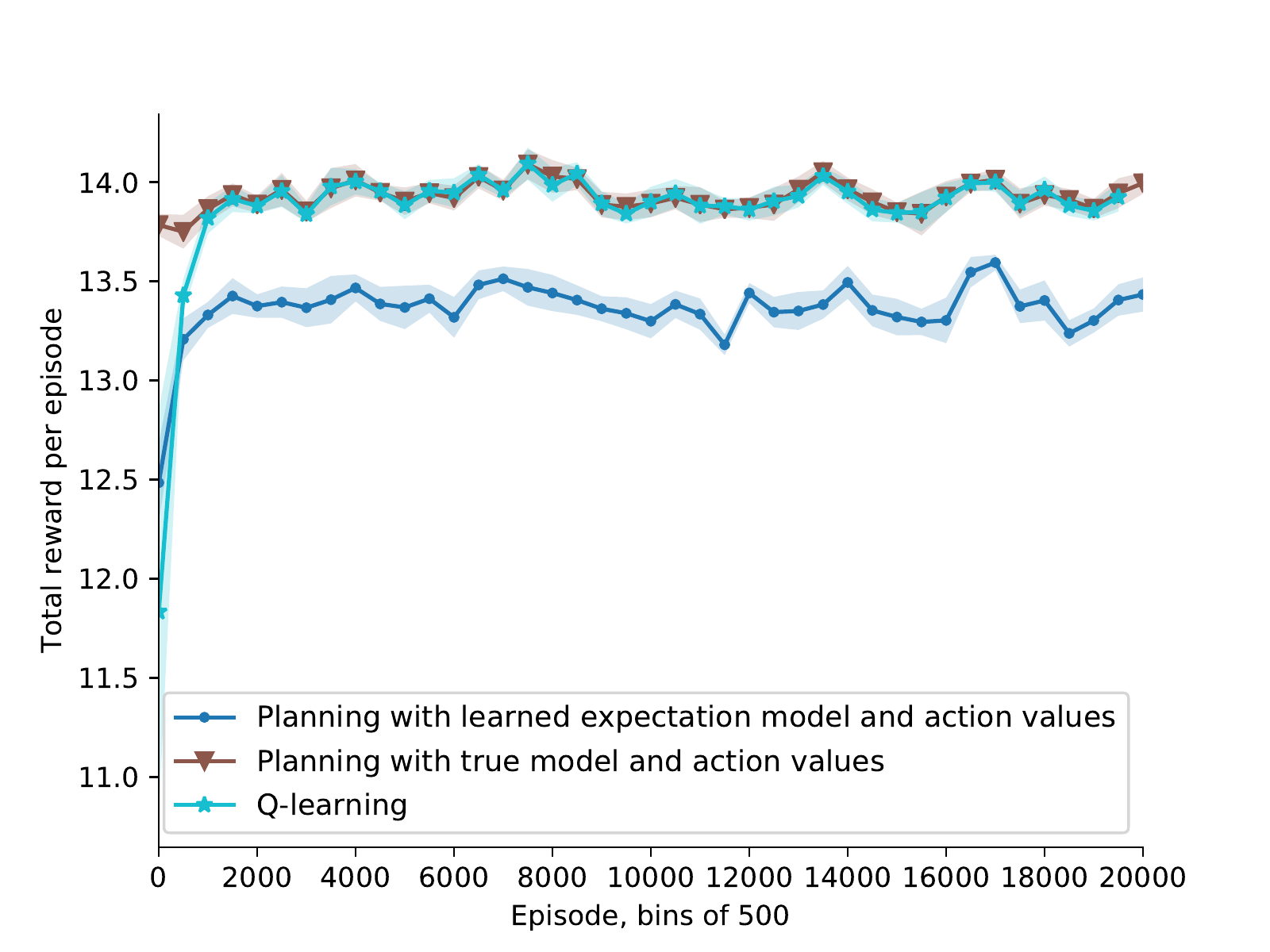}
\caption{{Q-learning and planning with expectation model and action values on the corridor stochastic environment with function approximation.
Feature-vectors size $d=14$ are random binary feature vectors, for each state. They all have the same $k$ number of $1$s, $k = 5$, picked at random, without replacement. 
Each curve is an average of 10 runs; each point represents the average number of total reward per episode, averaged over all the runs and over temporal stretches of $500$-episode bins.}}
\label{fig:long_run}
\vskip -0.2in
\end{figure}

\textbf{Algorithm 1.} One way in which planning can affect action is to compute a \emph{back-up value} with a model at the decision time for each of the actions:
\begin{equation}
   b(\s_t,a, \bw) \doteq r(\s_t,a) + \bar \s(\s_t,a)^\top\bw 
\end{equation}
\normalsize
and then act greedy w.r.t.~the model:
\begin{equation}
   A_t \doteq \argmax_{a\in \Actions} b(\s_t, a, \bw) 
\end{equation}
\normalsize
At planning time we compute the target as in (\ref{eq:compute_target})
and then use the TD update with $\s$ from the backup distribution and $\bar \s$ as an output from the model as in (\ref{eq:AVI_distr_model_weights}).
In this approach the backup values play the same role as action-values in greedy action selection
with an action-value method.
The equation above is greedy action selection with respect to the model; in practice we may want to change it to $\epsilon$-greedy action selection (in which with a small probability $\epsilon$
a random action is selected among all the actions with equal probability, independently of the value function estimates)
or to another method that ensures exploration.

\textbf{Algorithm 2.} 
Another approach to influence action with planning is to maintain action values
as cached backup values.
Planning already computes backed-up values as part of its normal operation.
Some of those values can be stored in the action-value approximator---a separate function approximator with weight vector $\bw_q$.
At decision time, actions are taken using $A_t \doteq \argmax_{a\in \Actions} \hat q(\s_t, a, \bw)$.
During planning time the same planning update as in (\ref{eq:AVI_distr_model_weights}) is performed
and in addition, action-value function parameters $\bw_q$ are updated:
\fontsize{9}{6}
\begin{align}
    \bw_q \leftarrow \bw_q + \alpha_{\bw_q} \bigg(b(\s,a, \bw) - \hat q(\s, a, \bw)\bigg) \nabla \hat q(\s, a, \bw)
    \label{eq:alg2_update}
\end{align}
\normalsize

\textbf{Algorithm 3.} 
The third approach is similar to Algorithm 2, but instead of caching into action values it caches all the way to an approximate policy.
Actions are sampled from the parameterized policy $\pi_{\bth}$ (e.g., a softmax policy) at decision time and the policy is updated as in policy gradient methods.
During planning time the same planning update as in (\ref{eq:AVI_distr_model_weights} and \ref{eq:compute_target}) are performed and the policy parameters are updated as:
\begin{align}
    &\delta(\s, \bw) \leftarrow r(\s,a) + \gamma \bar \s(\s,a)^\top\bw - \s^\top\bw \\
    &\bth \leftarrow \bth + \alpha_{\bth} \delta(\s, \bw) \nabla \big(\log \pi_{\bth}(\s,a)\big)
\end{align}

\section{Empirical Illustrations}
\label{sec:illustrations_3_methods}
This section reports some preliminary experimental results that illustrate differences between Q-learning,
planning with expectation model and action values, and the three proposed algorithms
that plan with expectation model and state values.
The goal of this section is to demonstrate empirically the proposed
approached of affecting action during planning.

\textbf{Domain}. 
In the first set of experiments, the domain used was a corridor shown in Figure \ref{fig:corridor} and described in \ref{sec:planning_w_action_values}. 
The agent's actions are stochastic:
with probability $9/10$, the actions would cause the agent to move one cell in the
corresponding direction, and with probability $1/10$, the agent would move instead in
the other direction.
In addition to stochastic actions, we introduce a non-stationarity 
to the domain inspired by van Seijen et al.~(2020): the goal location changes every $k$ episodes---\emph{a phase}---and the other terminal state becomes the goal state.

\begin{figure}[t!]
  \centering
  \subfigure[]
  {\includegraphics[scale=0.245]{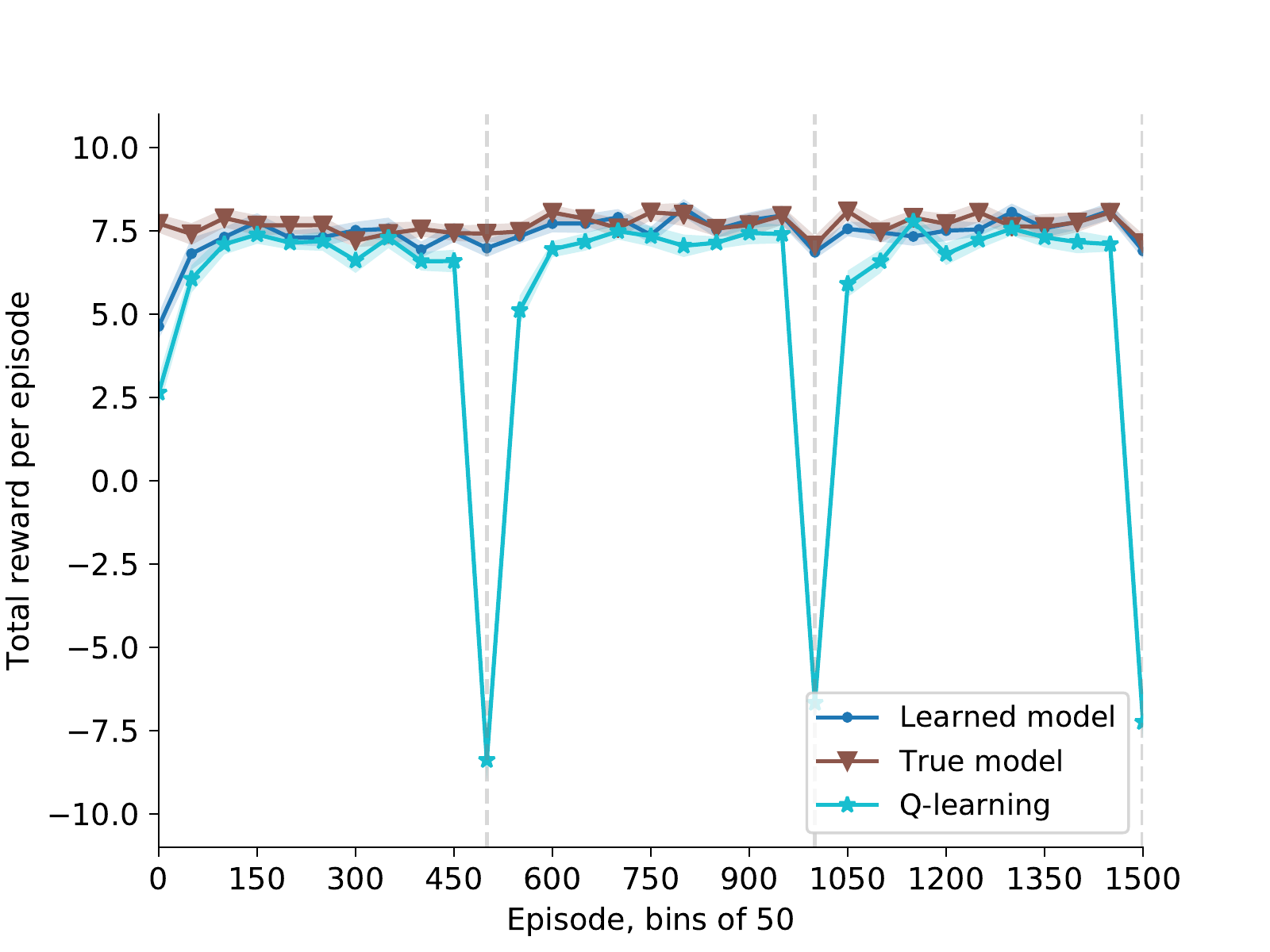}
  \label{fig:stoch_tabular}}
  \subfigure[]
  {\includegraphics[scale=0.245]{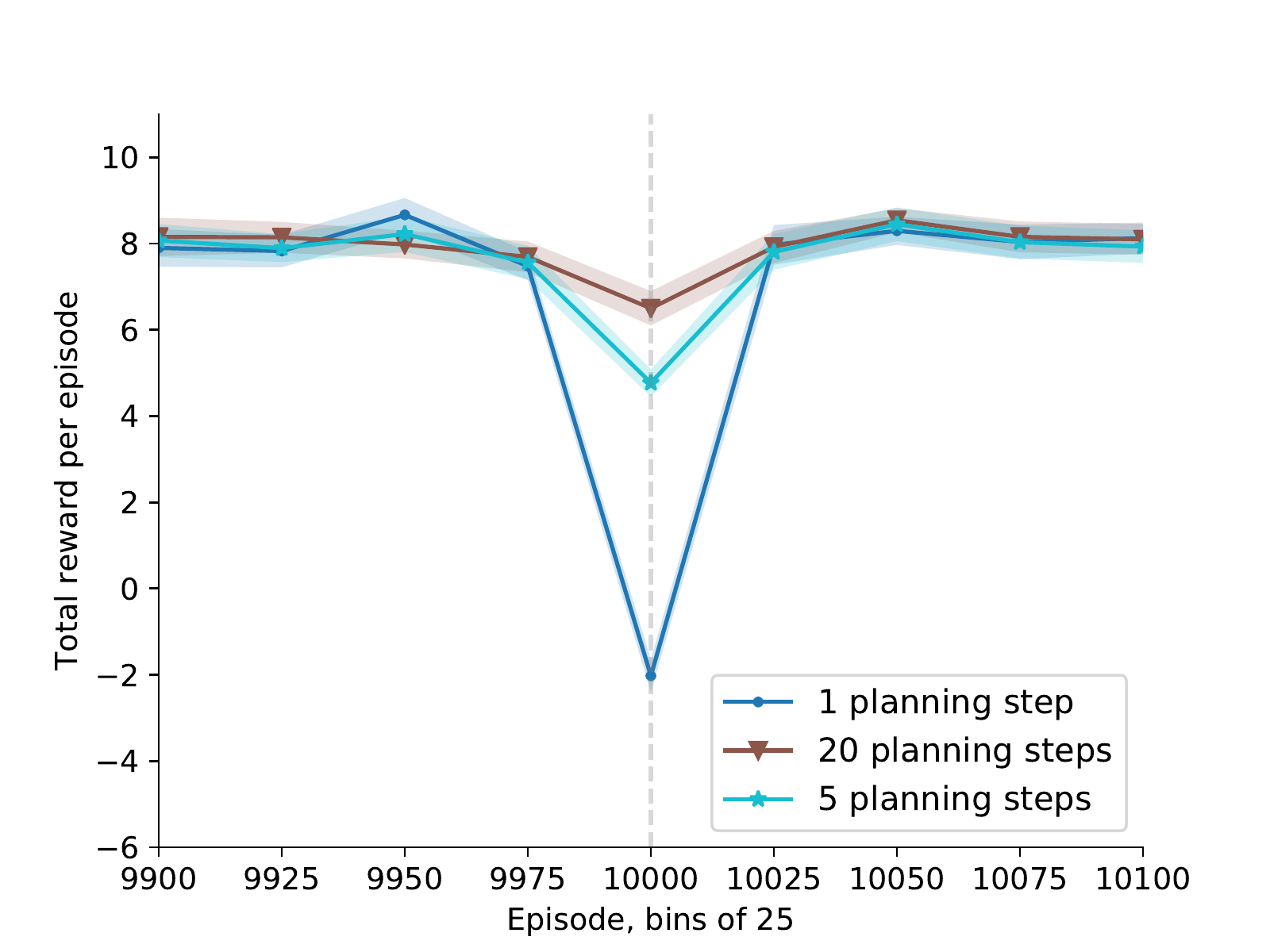}
  \label{fig:stoch_planning_steps}}
  \caption{Stochastic corridor results with tabular regime. (a) 
  show one-step Q-planning with the true environment model, a learned expectation model, and one-step Q-learning on a stochastic corridor.
  A goal switching phase is $500$ episodes. 
  (b) illustrate the effect of the number of planning steps when adjusting to the change in the environment. A goal switching phase is $10000$ episodes. 
  In (a) and (b): Each curve is an average of 30 runs; each point represents the average number of total reward per episode, averaged over all the runs and over temporal stretches of $50$-episode and $20$-episode bins for (a) and (b) respectively.}
  \label{fig:stochastic_corridor_non_stationary_impact}
  \vskip -0.2in
\end{figure}

\textbf{Illustration of non-stationarity impact.}
We first demonstrate the impact of non-stationary in fully-observable settings using the corridor domain (Figure~\ref{fig:corridor}) where the goal switches every $500$ episodes.
We compared the performance of three algorithms: (a) one-step Q-planning 
with the true environment model, (b) with a learned expectation model, (c) one-step Q-learning.
The first experiment was on a deterministic version of the corridor environment in which the actions were not stochastic. Figure \ref{fig:stochastic_corridor_non_stationary_impact} shows how the performance of Q-learning took a relatively long time to find the goal state consistently every time the goal got switched (i.e., a phase changed). The two model-based algorithms resulted in a small dip in performance when the goal switched, but they found the optimal policy relatively quickly by planning.

\begin{figure*}[t!]
\begin{center}
  \subfigure []
  {\includegraphics [scale=0.45]{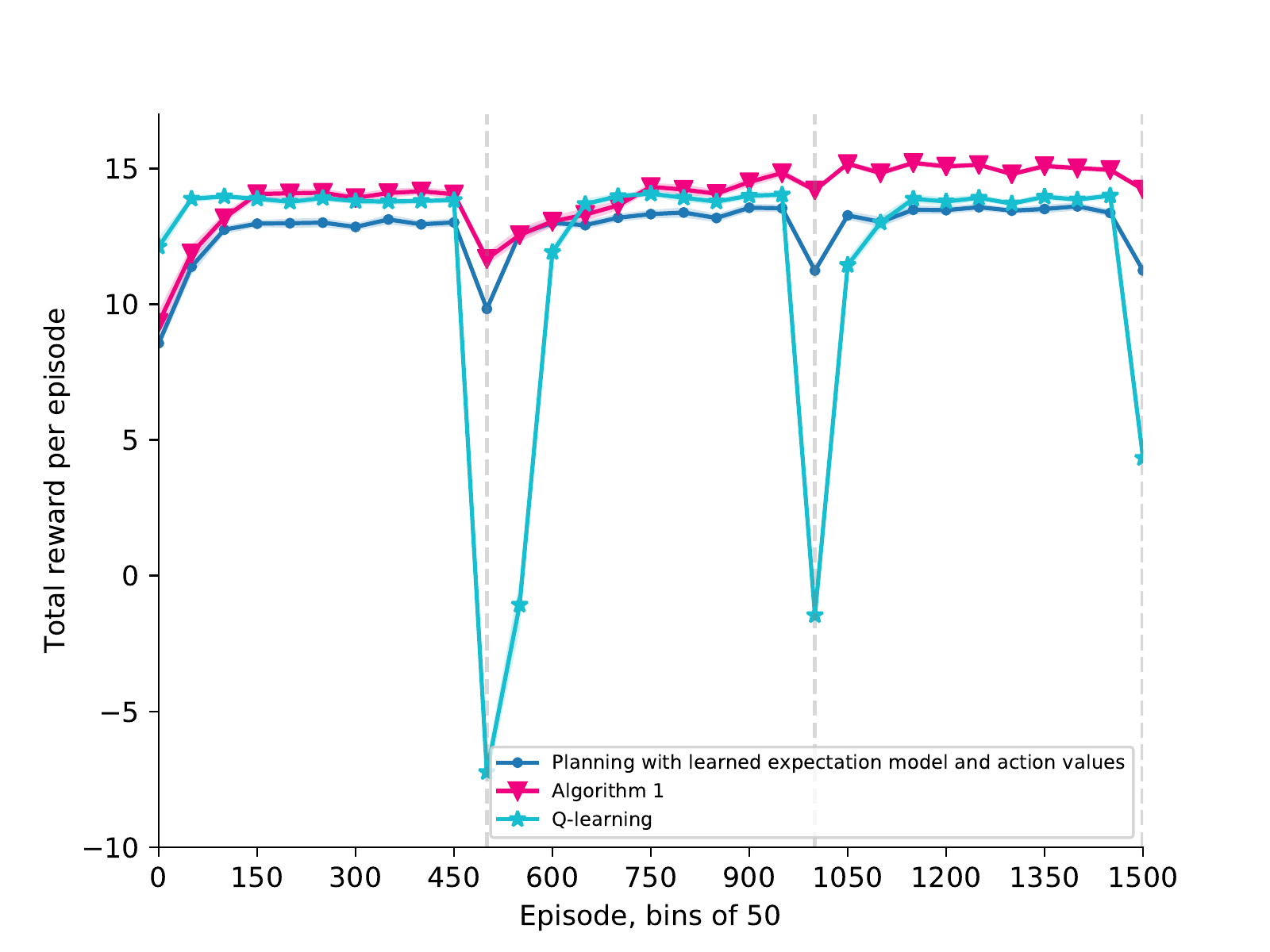}
  \label{fig:alg1}}
  \subfigure[]
  {\includegraphics[scale=0.45]{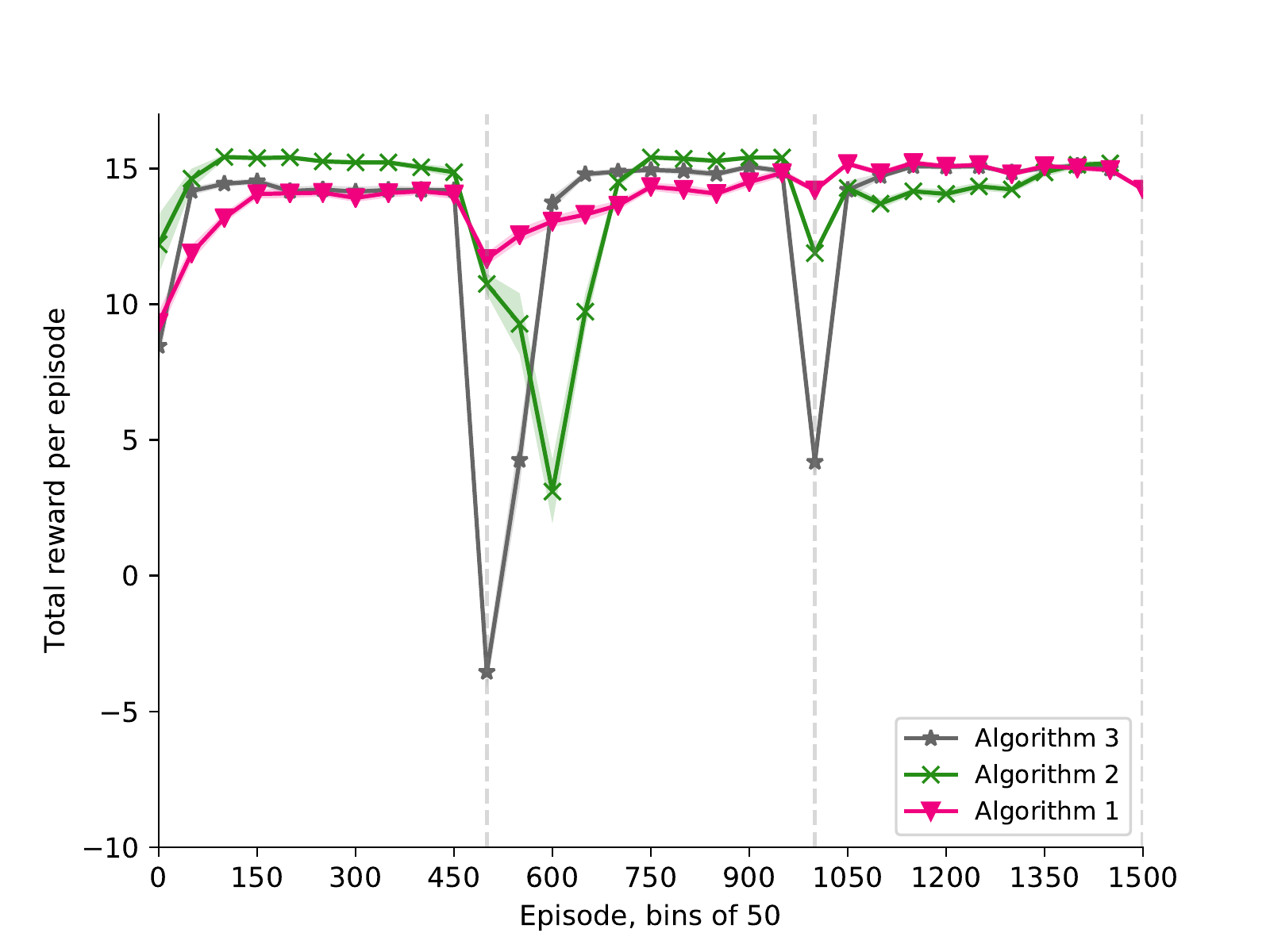}
  \label{fig:all_3_algo}}
  \caption{An illustration of the three proposed algorithms on the corridor domain. Each curve is an average of 30 runs; each point represents the average number of total reward per episode, averaged over all the runs and over temporal stretches of $50$-episode bins.
}
  \label{fig:3_fig}
  \vskip -0.2in 
\end{center}
\end{figure*}

In all cases, at the beginning of the second phase the agent starts from an adversarial value function that induces the agent to move only in one direction.
At the beginning of the third phase the agent no longer starts from a completely adversarial value function because it is not the optimal value function from the previous phase, which explains shorter recovery in phase $3$.

The higher the number of planning steps, the more efficient the agent is at the recovery.
Figure~\ref{fig:stoch_planning_steps} shows the effect of the number of planning steps 
on the recovery length and the impact on the drop.
It is clear that in a stochastic non-stationary environment planning  benefits the agent's performance. % TODO to clarify which one
We further illustrate our three proposed algorithms in this
non-stationary domain in the setting of function approximation.

\textbf{Illustration of proposed algorithms.}
In this experiment on the corridor domain, the feature vectors for each state are random binary feature vectors of size $d=14$.
They all have the same $k=3$ number of $1$s, picked at random, without replacement. 
The phases were $500$ episodes.

The first two algorithms use $\epsilon$-greedy action during decision time.
The agents' performance was measured in terms of the total reward per episode achieved,
and shown over temporal stretches of $n$-episode bins (see figure captions).
The model is learned online, in addition to the batch update from ER,
minimizing mean squared error as described in Section~\ref{sec:3_algorithms}. 

Figure~\ref{fig:alg1} shows the learning curves of a Q-learning agent, a learned expectation model with action values, and a learned expectation model with state values (Algorithm 1).
When the goal is switched, Q-learning performance dropped to about $-7$ between phases 1 and 2.
The impact of the goal change diminished as the Q-learning agent learned about the opposite goal location in phase 2. 
However, its performance still dropped to about $4$ at the end of phase 3.
Planning with expectation model and action values did not perform well as expected (demonstrated earlier in Figure~\ref{fig:long_run}).
Expectation model with action values recovered faster from the goal change than the Q-learning agent and its drop was not that significant, however, the overall performance is worse
because of the incorrect update (described in Section \ref{sec:planning_w_action_values}).
Algorithm 1 performed well overall and had the shortest recovery
at every phase when the goal changed.

Figure~\ref{fig:all_3_algo} shows the three proposed algorithms on the corridor domain.
The results for Algorithm 2 are less clear.
Algorithm 2 took longer to recover from the goal change than Algorithm 1.
Algorithm 3 was also more sensitive to goal change which may be because
policy gradient methods get stuck easily in local optimas.
Yet, Algorithm 3 recovered faster than Q-learning and the drop was not as large.

\section{Discussion}

A disadvantage of Algorithm 1 is that it involves significant computation
at decision time (computing one-step backups for all actions).
This could add significant latency to the agent's responses during decision time, especially if the number of actions to be considered is large.
Algorithms 2 and 3 mitigate this issue to an extent thanks to an explicit action-value function and a policy respectively.
One strength of Algorithm 2 compared to Algorithm 1 is its flexibility with regards to when the action values are updated---for which states and actions (\ref{eq:alg2_update}) is performed.
Another strength is that the update (\ref{eq:alg2_update}) can be done at planning time
rather than at decision time, which is less costly in terms of latency.
Algorithm 3 has a similar flexibility in terms of updating the policy. 
On the flip side, Algorithms 2 and 3 have the higher number of parameters compared to Algorithm 1.

All kinds of models require a backup distribution to sample states from which planning backups will be performed. 
There are many choices available
(e.g., see Pan et al., 2019). A general choice can be to learn a sample model for the backups. Note that this sample model for backups only needs to produce a state, which makes it much simpler than a sample model of transitions that takes a state and action as input and produces a sample of the next state and reward.

Finally, note that with expectation models, we do not expect to perform multi-step rollout-style backups. This is because the expected next state may not correspond to any real observation; Jafferfee (2020) showed the planning from such states can be detrimental.
In contrast, sample models are amenable for multi-step rollout-style backups. However, this can also be detrimental to planning because of compounding errors (Talvitie, 2014; 2017).
Furthermore,
expectation models may require less parameter tuning than sample or distribution models.

\section{Conclusion}
Many problems considered in AI have the nature of non-stationarity and dynamically changing goals
paired with some level of stochasticity in a domain.
Expectation models have their benefits compared to sample models:
they are expected to have less variance and hence result in quicker learning,
their space and computational requirements are relatively smaller.
We showed in this paper that planning with expectation models must update state values.
The studies reported here represent general settings of MBRL with expectation models that
are independent of many choices such as that of state representation.
Interesting areas for future work lie in a few directions:
1) how different kinds of backup distributions affect planning;
2) how the degree of function approximation affects planning; and
3) comparison of one-step backups in expectation models and iterated backups in sample models, and
4) extension of the theory in this paper to the case of non-linear value functions.

\section*{Acknowledgements}
The authors were supported by DeepMind, NSERC, and
CIFAR. Katya Kudashkina was supported
by the Arrell Food Institute.
The authors also wish to thank the following people for the ideas, discussions, and support: Eric Graves, Graham Taylor, Jakob Foerster, Julie Vale, John D. Martin, Joseph Modayil, Matthew Schlegel, and Zach Holland.
Resources used in preparing this research were provided, in part, by the Province of Ontario, the Government of Canada through CIFAR, companies sponsoring the Vector Institute, Compute Ontario, and Compute Canada. 

\section*{References}

\parskip=5pt
\parindent=0pt
\def\hangin{\hangindent=0.2in}
\def\bibitem[#1]#2{\hangin}

\hangin
Abbeel, P., Coates, A., Quigley, M., and Ng, A. Y.~(2007). An application of reinforcement learning to aerobatic helicopter flight. In \emph{Proceedings of the 21st Conference on Neural Information Processing Systems}. pp. 1--8.

\hangin
Albus, J. S. (1971). A theory of cerebellar function.~\emph{Mathematical biosciences}, 10(1-2):25–61.

\hangin
Albus, J. S. (1981)~\emph{Brains, behavior, and robotics}. Byte Books, Peterborough, NH.

\hangin
Atkeson, C. G. and Santamaria, J. C. (1997). A comparison of direct and model-based reinforcement
learning. In~\emph{Proceedings of International Conference on Robotics and Automation}. vol. 4, pp. 3557--3564.

\hangin
Barto, A. G., Sutton, R. S., and Anderson, C. W. (1983). Neuronlike adaptive elements that can solve difficult learning control problems.~\emph{IEEE transactions on systems, man, and cybernetics}, SMC-13(5):834--846.

\hangin
Bellman, R. (1957).~\emph{Dynamic Programming}. Princeton University Press.

\hangin
Bertsekas, D. P. (2012).~\emph{Dynamic Programming and Optimal Control, Volume II:
Approximate Dynamic Programming}, fourth edition. Athena Scientific, Belmont, MA.

\hangin
Bertsekas, D. P., Tsitsiklis, J. N. (1996)~\emph{Neuro-Dynamic Programming}. Athena Scientific,
Belmont, MA.

\hangin
Browne, C. B., Powley, E., Whitehouse, D., Lucas, S. M., Cowling, P. I., Rohlfshagen, P., Tavener, S., Perez, D., Samothrakis, S., and Colton, S. (2012). A survey of Monte Carlo tree search methods. \emph{IEEE Transactions on Computational Intelligence and AI in games}, 4(1):1--43.

\hangin
Buckman, J., Hafner, D., Tucker, G., Brevdo, E., and Lee, H. (2018). Sample-Efficient Reinforcement Learning with Stochastic Ensemble Value Expansion. In~\emph{Proceedings of the 32nd Conference on Neural Information Processing Systems}, pp. 8224--8234.

\hangin
Cassandra, A. R., Kaelbling, L. P., and Littman, M. L. (1994). Acting optimally in partially observable stochastic domains. In~\emph{AAAI, Vol. 94}, pp. 1023--1028.

\hangin
Chapman, D. and Kaelbling, L. P. (1991). Input generalization in delayed reinforcement learning:
An algorithm and performance comparisons. In \emph{Proceedings of the Twelfth International Conference on Artificial Intelligence}, pp. 726--731. Morgan Kaufmann, San Mateo, CA.

\hangin
Chua, K., Calandra, R., McAllister, R., and Levine, S. Deep reinforcement learning in a handful of trials using probabilistic dynamics models. In \emph{Proceedings of the 31st Conference on Neural Information Processing Systems}, pp. pp. 4754–4765, Montr{\'e}al, Canada. Curran Associates, Inc. 2018.

\hangin
Deisenroth, M., Rasmussen, C. E. (2011). PILCO: A model-based and data-efficient approach to policy search. In \emph{Proceedings of the 28th International Conference on Machine Learning} (pp.~465--472). Omnipress.

\hangin
Doll, B. B., Simon, D. A., Daw, N. D. (2012). The ubiquity of model-based reinforcement
learning. \emph{Current Opinion in Neurobiology, 22}(6):1--7.

\hangin
Farley, B. and Clark, W. (1954). Simulation of self-organizing systems by digital computer.~\emph{Transactions of the IRE Professional Group on Information Theory}, 4(4):76–84.

\hangin
Feinberg, V., Wan, A., Stoica, I., Jordan, M. I., Gonzalez, J. E., and Levine, S. (2018).
Model-based value estimation for efficient model-free reinforcement learning. ArXiv:1803.00101.

\hangin
Finnsson, H. and Bj{\"o}rnsson, Y. (2008). Simulation-Based Approach to General Game Playing. In \emph{AAAI}, volume 8, pp.259--264.

\hangin
Gu, S., Lillicrap, T., Sutskever, I., and Levine, S. (2016). Continuous deep q-learning with model-based acceleration. In \emph{Proceedings of the 33rd International Conference on International Conference on Machine Learning, Volume 48}, JMLR.org.

\hangin
Ha, D. and Schmidhuber, J. World models. (2018). In~\emph{Proceedings of the 32nd Conference on Neural Information Processing Systems}, pp. 2451--2463.

\hangin
Harmon, M. E. and Baird III, L. C. (1996). Spurious Solutions to the Bellman Equation. Technical report, Citeseer.

\hangin
Hessel, M., Modayil, J., van Hasselt, H., Schaul, T., Ostrovski, G., Dabney, W., Horgan, D., Piot, B., Azar, M. G., and Silver, D. (2018). Rainbow: Combining Improvements in Deep Reinforcement Learning. In \emph{AAAI}, pp. 3215--3222.

\hangin
Hester, T. and Stone, P. (2011). Learning and using models. Wiering, M. and van Otterlo, M. (eds.), \emph{Reinforcement Learning: State of the Art}. Springer Verlag, Berlin, Germany.

\hangin
Hochreiter, S. and Schmidhuber, J. (1997). Long Short-Term Memory. \emph{Neural Computation, 9}(8):1735--1780.

\hangin
Holland, G. Z., Talvitie, E. J., and Bowling, M. (2018). The Effect of Planning Shape on Dyna-style
Planning in High-dimensional State Spaces. ArXiv:1806.01825.

\hangin
Jafferjee, T. (2020). Chasing Hallucinated Value: A Pitfall of Dyna Style Algorithms with Imperfect Environment Models. Master’s thesis, The University of Alberta.

\hangin
Kaelbling, L. P. (1993a). Hierarchical learning in stochastic domains: Preliminary results. In \emph{Proceedings of the 10th International Conference on Machine Learning}, pp. 167--173. Morgan Kaufmann.

\hangin
Kaelbling, L. P. (1993b). Learning to achieve goals. In \emph{Proceedings of the Thirteenth International Joint Conference on Artificial Intelligence}, pp. 1094--1099.

\hangin
Kaiser, L., Babaeizadeh, M., Milos, P., Osinski, B., Camp- 359 bell, R. H., Czechowski, K., Erhan, D., Finn, C., Kozakowski, P., Levine, S., Mohiuddin, A., Sepassi, R., Tucker, G., and Michalewski, H. (2019). Model-based reinforcement learning for Atari. ArXiv:1903.00374.

\hangin
Kingma, D. P. and Welling, M. (2014). Auto-Encoding Variational Bayes. In \emph{Proceedings of the 2nd International Conference on Learning Representations}.

\hangin
Kurutach, T., Clavera, I., Duan, Y., Tamar, A., and Abbeel, P. (2018). Model-ensemble trust-region policy optimization. In \emph{Proceedings of the 6th International Conference on Learning Representations}.

\hangin
Kuvayev, L. and Sutton, R. S. (1996). Model-based reinforcement learning with an approximate, learned model. In \emph{Proceedings of the Ninth Yale Workshop on Adaptive and Learning Systems}, pp. 101--105. Citeseer.

\hangin
Leibfried, F., Kushman, N., and Hofmann, K. (2017). A Deep Learning Approach for Joint Video Frame and Reward Prediction in Atari Games. In \emph{Workshop on Principled Approaches to Deep Learning, International Conference on Machine Learning}, Sydney, Australia.

\hangin
Lillicrap, T. P., Hunt, J. J., Pritzel, A., Heess, N., Erez, T., Tassa, Y., Silver, D., and Wierstra, D. (2016). Continuous control with deep reinforcement learning. In \emph{Proceedings of the 4th International Conference on Learning Representations (Poster)}.

\hangin
Lin, L.-J. (1992). Self-improving reactive agents based on reinforcement learning, planning and teaching. \emph{Machine learning}, 8(3-4):293--321.

\hangin
Littman, M. L., Sutton, R. S., Singh. (2002). Predictive representations of state. In~\emph{Proceedings of the 14th Conference on Neural Information Processing Systems}, pp. 1555--1561. MIT Press,
Cambridge, MA.

\hangin
Mnih, V., Kavukcuoglu, K., Silver, D., Graves, A., Antonoglou, I., Wierstra, D., et al.~(2013).~Playing
Atari with Deep Reinforcement Learning.~In \emph{Deep Learning Workshop, Neural Information Processing Systems}.

\hangin
Mnih, V., Kavukcuoglu, K., Silver, D., Rusu, A. A., Veness, J., Bellemare, M. G., Graves, A., Riedmiller, M., Fidjeland, A. K., Ostrovski, G., Petersen, S., Beattie, C., Sadik, A., Antonoglou, I., King, H., Kumaran, D., Wierstra, D., Legg, S., Hassabis, D. (2015). Human-level control through deep reinforcement learning. \emph{Nature}, 518(7540):529--533.

\hangin
Moore, A. W., Atkeson, C. G. (1993). Prioritized sweeping: Reinforcement learning with less
data and less real time.~\emph{Machine Learning}, 13(1):103--130.

\hangin
Naik, A., Shariff, R., Yasui, N., and Sutton, R. S. (2019). Discounted reinforcement learning is not an optimization problem. \emph{Optimization Foundations of Reinforcement Learning Workshop at Neural Information Processing Systems}.

\hangin
Oh, J., Guo, X., Lee, H., Lewis, R. L., Singh, S. (2015). Action-conditional video prediction using deep networks in Atari games. In \emph{Proceedings of the 28th Conference on Neural Information Processing Systems}, pp. 2845--2853. Curran Associates, Inc.

\hangin
Pan Y., Zaheer M., White A., Patterson A., and White M. (2018). Organizing experience: a deeper look at replay mechanisms for sample-based planning in continuous state domains. In \emph{Proceedings of the 27th International Joint Conference on Artificial Intelligence}. AAAI Press, 4794--4800.

\hangin
Pan Y., Yao H., Farahmand A.M., White M. (2019). Hill Climbing on Value Estimates for Search-control in Dyna. In \emph{Proceedings of the 28th International Joint Conference on Artificial Intelligence}.

\hangin
Peng, J. and Williams, R. J. (1993). Efficient learning and planning within the Dyna framework. \emph{Adaptive behavior}, 1(4): 437--454.

\hangin
Rezende D.J., Mohamed S., Wierstra D. (2014). Stochastic backpropagation and
approximate inference in deep generative models. \emph{Proceedings of the 31st International Conference on Machine Learning}.

\hangin
Ring, M. B. (1994). \emph{Continual learning in reinforcement environments}. PhD thesis, University of Texas at Austin.

\hangin
Schaul, T., Horgan, D., Gregor, K., \& Silver, D. (2015). Universal value function approximators. In \emph{Proceedings of the 32nd International Conference on Machine
Learning}, pp. 1312--1320.  JMLR: W\&CP volume 37.

\hangin
Schaul, T., Quan, J., Antonoglou, I., and Silver, D. (2016). Prioritized experience replay. In \emph{Proceedings of the 4th International Conference on Learning Representations (Poster)}.

\hangin
Schmidhuber, J. (2015). Deep learning in neural networks: An overview. \emph{Neural Networks}, 6 :85--117.

\hangin
Schrittwieser, J., Antonoglou, I., Hubert, T., Simonyan, K., Sifre, L., Schmitt, S., Guez, A., Lockhart, E., Hassabis, D., Graepel, T., Lillicrap, T., and Silver, D. (2019). Mastering Atari, Go, chess and shogi by planning with a learned model. In \emph{Proceedings of the 33rd Conference on Neural Information Processing Systems}, Vancouver, Canada.

\hangin
Shariff, R.; and Szepesv{\'a}ri, C. (2020). Efficient Planning in Large MDPs with Weak Linear Function Approximation.  In \emph{Proceedings of the 34th Conference on Neural Information Processing Systems}, Vancouver, Canada.

\hangin
Silver, D., Lever, G., Heess, N., Degris, T., Wierstra, D., Riedmiller, M. (2014). Deterministic policy gradient algorithms. In \emph{Proceedings of the 31st International Conference on Machine Learning}, pp. 387--395.

\hangin
Silver, D., Schrittwieser, J., Simonyan, K., Antonoglou, I., Huang, A., Guez, A., Hubert, T., Baker, L., Lai, M., Bolton, A., Chen, Y., Lillicrap, L., Hui, F., Sifre, L., van den Driessche, G., Graepel, T., Hassibis, D. (2017). Mastering the game of Go without human knowledge. \emph{Nature}, 550(7676):354--359.

\hangin
Singh, S. P. Reinforcement learning with a hierarchy of abstract models. (1992). In \emph{Proceedings of the National Conference on Artificial Intelligence}, 10, p.~202.

\hangin
Singh, S. P., Jaakkola, T., \& Jordan, M. I. (1995). Reinforcement learning with soft state aggregation. In \emph{Proceedings of the 8th Conference on Neural Information Processing Systems}, pp. 361-368.

\hangin
Sorg, J., Singh, S. Linear options. (2010). In \emph{Proceedings of the 9th International Conference on
Autonomous Agents and Multiagent Systems}, pp. 31--38.

\hangin
Sugiyama, M., Hachiya, H., Morimura, T. (2013). \emph{Statistical Reinforcement Learning: Modern Machine Learning Approaches}. Chapman \& Hall/CRC.

\hangin
Sutton, R. S. (1988). Learning to predict by the methods of temporal differences.~\emph{Machine learning}, 3(1): 9–44.

\hangin
Sutton, R. S. (1990). Integrated architectures for learning, planning, and reacting based on approximating dynamic programming. In \emph{Proceedings of the 7th International Conference on Machine Learning }, pp. 216--224. Morgan Kaufmann.

\hangin
Sutton, R. S. (1991). Dyna, an integrated architecture for learning, planning, and reacting.
\emph{SIGART Bulletin}, 2(4):160–163. ACM, New York.

\hangin
Sutton, R. S. and Pinette, B. (1985). The learning of world models by connectionist networks. In
\emph{Proceedings of the 7th Annual Conference of the Cognitive Science Society}. pp. 54--64.

\hangin
Sutton, R.S., Precup, D., Singh, S. (1998). Intra-option learning about temporally abstract actions. \emph{Proceedings of the 15th International Conference on Machine Learning}, pp. 556-564. Morgan Kaufmann.

\hangin
Sutton, R. S. , Rafols, E. J. , and Koop, A. (2005). Temporal Abstraction in Temporal-difference Networks. In \emph{Proceedings of the 19th Conference on Neural Information Processing Systems}.

\hangin
Sutton, R. S., Szepesv{\'a}ri, Cs., Geramifard, A., Bowling, M. (2008). Dyna-style planning with linear function approximation and prioritized sweeping, In \emph{Proceedings of the 24th Conference on Uncertainty in Artificial Intelligence}, pp. 528--536.

\hangin
Sutton, R. S. and Barto, A. G. (1981). An adaptive network that constructs and uses and internal model of its world. \emph{Cognition and Brain Theory 4}, 217--246.

\hangin
Sutton, R. S. and Barto, A. G. (2018).~\emph{Reinforcement learning: An Introduction}. MIT press.

\hangin
Sutton, R. S. (2019). \textit{The Bitter Lesson} [Blog post]. Retrieved from http://www.incompleteideas.net/IncIdeas\\/BitterLesson.html.

\hangin
Talvitie, E. (2014). Model Regularization for Stable Sample Rollouts. In \emph{Proceedings of the 30th Conference on Uncertainty in Artificial Intelligence}, pp. 780--789.

\hangin
Talvitie, E. (2017). Self-correcting Models for Model-based Reinforcement Learning. In \emph{Proceedings of the 31st AAAI Conference on Artificial Intelligence}.

\hangin
van Hasselt, H. (2010). Double Q-learning. In \emph{Proceedings of the 24th Conference on Neural Information Processing Systems}, pp.~2613--2621.

\hangin
van Hasselt, H., Guez, A., and Silver, D. (2016). Deep Reinforcement 343 Learning with Double Q-Learning. In \emph{Proceedings of the Thirtieth AAAI Conference on Artificial Intelligence}, AAAI Press.

\hangin
van Hasselt, H. P., Hessel, M., and Aslanides, J. (2019). When to use parametric models in reinforcement learning? In \emph{Proceedings of the 33rd Conference on Neural Information Processing Systems}, pp. 14322--14333.

\hangin
van Seijen, H., Nekoei, H., Racah, E., and Chandar, S. (2020). The LoCA Regret: A Consistent Metric to Evaluate Model-Based Behavior in Reinforcement Learning. In \emph{Proceedings of the 34th Conference on Neural Information Processing Systems}.

\hangin
Wan, Y., Zaheer, M., White, A., White, M., and Sutton, R. S. (2019). Planning with expectation models. In \emph{Proceedings of the 28th International Joint Conference on Artificial Intelligence}.

\hangin
Watkins, C. (1989).\emph{Learning from delayed rewards.} PhD thesis, King’s College, Cambridge, England.

\hangin
Weber, T., Racani{\`e}re, S, Reichert, D.P., Buesing, L., Guez, A., Rezende, D., Badia, A.P., Vinyals, O., Heess, N., Li Y, Pascanu, R. (2017). Imagination-augmented Agents for Deep Reinforcement Learning. \emph{Proceedings of the 31st Conference on Neural Information Processing Systems}, Long Beach, CA, USA. Curran Associates Inc.

\hangin
Wiering, M., Salustowicz, R., and Schmidhuber, J. (2001). Model-based reinforcement learning for evolving soccer strategies. In \emph{Computational intelligence in games}, pp. 99--132. Springer.

\hangin
White A. (2015). \emph{Developing a predictive approach to knowledge}. PhD thesis, University of Alberta.

\hangin
Yao, H., Bhatnagar, S., and Diao, D. (2009). Multi-step linear Dyna-style planning. In \emph{Proceedings of the 22nd International Conference on Neural Information Processing Systems}, pp. 2187--2195. Curran Associates Inc.

\clearpage
\appendix

\section{Algorithms}

\removelatexerror
\begin{center}

\begin{algorithm}[H]
\DontPrintSemicolon
\SetAlgoLined
% \KwIn{State $\s$}
%   \small
\SetKwInput{AP}{Algorithm parameters}
\AP{step size $\alpha_{\bw}, \alpha_{\be}$}
Initialize state-value parameters $\bw$, expectation-model parameters $\be$\;
Compute state $\s$ with the state-update function $u$\;
 \While{still time to train}
 {
    % $b(\s_t,a, \bw) \leftarrow $\;
    $A = \argmax_{a\in \Actions} \big[r(\s,a) + \gamma \bar \s(\s,a)^\top\bw \big]$ \ \ \ {\small(lookahead)}\;
    Take action $A$, observe $R, O'$\;
    Compute state $\s'$ with the state-update function $u$\;
    $\delta \leftarrow R + \bw^\top \s' - \bw^\top \s$\;
    $\bw \leftarrow \bw + \alpha \delta 
    % \s$\;
    \nabla (\bw^\top \s)$\;
    Update model parameters $\be$.\;
    \While{still time to plan}
    {
        Sample state $\s$ from the backup distribution\;
        $g(\s, \bw)  = \max_{a \in \Actions} \big[ r(\s,a) + \gamma \bw^\top \bar{\s}(\s,a) \big]$\;
        $\bw \leftarrow \bw + \alpha_{\bw} (g(\s, \bw)  - \bw^\top \s)
        % \s$\;
        \nabla (\bw^\top \s)$\;
    }
    $\s = \s'$\;
 }
 \caption{Learning and planning with an expectation model; one-step lookahead at decision time}
 \label{algo:plan_v_lookahead}
\end{algorithm}

% \fontsize{5}{6}
\begin{algorithm}[H]
\DontPrintSemicolon
\SetAlgoLined
% \KwIn{State $\s$}
\SetKwInput{AP}{Algorithm parameters}
%   \small
\AP{step size $\alpha_{\bw}$, $\alpha_{\bw_q}, \alpha_{\be}$}
Initialize state-value parameters $\bw$,
action-value parameters $\bw_q $, expectation-model parameters $\be$\;
Compute state $\s$ with the state-update function  $u$\;
 \While{still time to train}
 {
    $A = \argmax_{a\in \Actions} \hat q(\s, a, \bw)$\;
    Take action $A$, observe $R, O'$\;
    Compute state $\s'$ with the state-update function $u$\;
    $\delta \leftarrow R + \bw^\top \s' - \bw^\top \s$\;
    $\bw \leftarrow \bw + \alpha \delta 
    % \s$\;
    \nabla (\bw^\top \s)$\;
    % \fontsize{7.5}{6}
    $\bw_q \leftarrow \bw_q + \alpha_{\bw_q} \bigg(\delta - \hat q(\s, a, \bw)\bigg) \nabla \hat q(\s, a, \bw)$\;
    \normalsize
    Update model parameters $\be$.\;
    \While{still time to plan}
    {
        Sample state $\s$ from the backup distribution\;
        $g(\s, \bw)  = \max_{a \in \Actions} \big[ r(\s,a) + \gamma \bw^\top \bar{\s}(\s,a) \big]$\;
        $\bw \leftarrow \bw + \alpha_{\bw} (g(\s, \bw)  - \bw^\top \s)
        % \s$\;
        \nabla (\bw^\top \s)$\;
        $\slash\slash$ Update cached action-value function\;
        $b(\s,a, \bw) \leftarrow r(\s,a) + \bar \s(\s,a)^\top\bw$\;
        \fontsize{7.5}{6}
        $\bw_q \leftarrow \bw_q + \alpha_{\bw_q} \bigg(b(\s,a, \bw) - \hat q(\s, a, \bw)\bigg) \nabla \hat q(\s, a, \bw)$\;
        \normalsize
    }
    $\s = \s'$\;
 }
 \caption{Act with action-value function that is updated during planning from backup values.}
 \label{algo:plan_q_lookahead_cached}
\end{algorithm}
\normalsize
\end{center}

\removelatexerror
\begin{algorithm*}[H]
\DontPrintSemicolon
% \small
\SetAlgoLined
%   \scriptsize
% \KwIn{State $\s$}
\SetKwInput{AP}{Algorithm parameters}
\AP{step size $\alpha_{\bw}$, $\alpha_{\theta}, \alpha_{\be}$}
Initialize state-value parameters $\bw$, 
policy parameters $\bth $,
expectation-model parameters $\be$\;
Compute state $\s$ with the state-update function $u$\;
 \While{still time to train}
 {
    $A \sim \pi_{\bth}$\;
    Take action $A$, observe $R, O'$\;
    Compute state $\s'$ with the state-update function $u$\;
    $\delta \leftarrow R + \bw^\top \s' - \bw^\top \s$\;
    $\bw \leftarrow \bw + \alpha_{\bw} \delta 
    % \s$\;
    \nabla (\bw^\top \s)$\;
    $\bth \leftarrow \bth + \alpha_{\bth} \delta \nabla \big(\log \pi_{\bth}(\s,a)\big)$\;
    Update model parameters $\be$.\;
    \While{still time to plan}
    {
        Sample state $\s$ from the backup distribution\;
        $g(\s, \bw)  = \max_{a \in \Actions} \big[ r(\s,a) + \gamma \bw^\top \bar{\s}(\s,a) \big]$\;
        $\bw \leftarrow \bw + \alpha_{\bw} (g(\s, \bw)  - \bw^\top \s) \nabla (\bw^\top \s)$\;
        $\slash\slash$ Update policy parameters\;
        {
            Sample action $a$ from $\pi_\theta(\s,\cdot)$\;
            $\delta \leftarrow r(\s,a) + \gamma \bar \s(\s,a)^\top\bw - \s^\top\bw $\;
            $\bth \leftarrow \bth + \alpha_{\bth} \delta \nabla \big(\log \pi_{\bth}(\s,a)\big)$\;
        }
    }
    $\s = \s'$\;
 }
 \caption{Act with an explicit parameterized policy.}
 \label{algo:plan_with_policy}
\end{algorithm*}
\normalsize

\clearpage
\onecolumn
\section{Parameters}
% \usepackage{multirow}
% \usepackage{graphicx}
In all experiments:
\begin{itemize}
\setlength\itemsep{-0.5em}
    \item The discount parameter $\gamma$ was $1$.
    \item All the initial weight parameters were initialized to zeros.
    \item The backup distribution in planning was from ER.
    \item $\epsilon$-greedy policy with $\epsilon= 0.1$.
\end{itemize}

We performed parameter studies during which we ran algorithms for $1000$ episodes. The results for each of the parameter sets were averaged based on 30 runs. The final parameters were chosen based on the best performing parameter set.

% ne-step Q-planning (Sutton & Barto, 2018,Chapter 10) with the true environment model, (b) with alearned expectation model, (c) one-step Q-learning (Watkinsand Dayan, 1992).

\subsection{Monster Counterexample, Figure 1}
\begin{table}[h]
% \resizebox{\textwidth}{!}
{%
\begin{tabular}{|l|c|c|}
\hline
                                   & Q-learning         & Planning with expectation model and action values \\ \hline
Stochasticity during decision time & \multicolumn{2}{c|}{0.5}                                               \\ \hline
Action-value function step size    & 0.3                & 0.3                                               \\ \hline
Model step size                    & \multirow{2}{*}{-} & 0.3                                               \\ \cline{1-1} \cline{3-3} 
Number of planning steps           &                    & 20                                                \\ \hline
Feature selection                  & \multicolumn{2}{c|}{One-hot encoding}                                  \\ \hline
\end{tabular}%
}
\end{table}

\subsection{Stochastic corridor, Figure 3}
\begin{table}[h]
% \resizebox{\textwidth}{!}
{%
\begin{tabular}{|l|c|c|}
\hline
                                   & Q-learning & Planning with expectation model and action values \\ \hline
Stochasticity during decision time & \multicolumn{2}{c|}{0.3}                                       \\ \hline
Action-value function step size    & 0.01       & 0.01                                              \\ \hline
Model step size                    & -          & 0.1                                               \\ \hline
Number of planning steps           &            & 20                                                \\ \hline
Feature selection & \multicolumn{2}{c|}{Feature-vector of size $d=14$ with $k=5$ random bits} \\ \hline
\end{tabular}%
}
\end{table}

\subsection{Stochastic corridor, Figure 4a}
% \usepackage{graphicx}
\begin{table}[h]
% \resizebox{\textwidth}{!}
{%
\begin{tabular}{|l|c|c|}
\hline
                                   & Q-learning & Planning with expectation model and action values \\ \hline
Stochasticity during decision time & \multicolumn{2}{c|}{0.3}                                       \\ \hline
Action-value function step size    & 0.1        & 0.1                                               \\ \hline
Model step size                    & -          & 0.1                                               \\ \hline
Number of planning steps           & -          & 20                                                \\ \hline
Feature selection                  & \multicolumn{2}{c|}{One-hot encoding}                          \\ \hline
Goal change                        & \multicolumn{2}{c|}{Every 500 episodes}                        \\ \hline
\end{tabular}%
}
\end{table}

\subsection{Stochastic corridor, Figure 4b}
\begin{table*}[!h]
% \resizebox{\textwidth}{!}{%
{
\begin{tabular}{|l|c|c|}
\hline
                                   & Q-learning & Planning with a true environment model \\ \hline
Stochasticity during decision time & \multicolumn{2}{c|}{0.3}                            \\ \hline
Action-value function step size    & 0.1        & 0.1                                    \\ \hline
Number of planning steps           &            & 20                                     \\ \hline
Feature selection                  & \multicolumn{2}{c|}{One-hot encoding}               \\ \hline
Goal change                        & \multicolumn{2}{c|}{Every 10000 episodes}           \\ \hline
\end{tabular}%
}
\end{table*}

\clearpage
\subsection{Stochastic corridor, Figure 5}
% \usepackage{multirow}
% \usepackage{graphicx}
\begin{table}[h]
% \resizebox{\textwidth}{!}
{%
\begin{tabular}{|l|c|c|c|c|c|}
\hline
\multirow{3}{*}{} & \multirow{3}{*}{Q-learning} & \multicolumn{4}{c|}{Planning with expectation model} \\ \cline{3-6} 
                                   &       & \multirow{2}{*}{with action values} & \multicolumn{3}{c|}{with state values}  \\ \cline{4-6} 
                                   &       &                                     & algorithm 1 & algorithm 2 & algorithm 3 \\ \hline
Stochasticity during decision time & \multicolumn{5}{c|}{0.1}                                                              \\ \hline
Action-value function step size    & 0.001 & 0.1                                 & -           & 0.01        & -           \\ \hline
Model step size                    & -     & 0.1                                 & 0.1         & 0.001       & 0.01        \\ \hline
Number of planning steps           &       & 5                                   &             &             &             \\ \hline
Feature selection                  & \multicolumn{5}{c|}{Feature-vector of size $d=14$ with $k=5$ random bits}             \\ \hline
State-value function step size     & \multicolumn{2}{c|}{\multirow{2}{*}{-}}     & 0.01        & 0.01        & 0.01        \\ \cline{1-1} \cline{4-6} 
Policy step size                   & \multicolumn{2}{c|}{}                       & -           & -           & 0.001       \\ \hline
Goal change                        & \multicolumn{5}{c|}{Every 500 episodes}                                               \\ \hline
\end{tabular}%
}
\end{table}
% \section{Do \emph{not} have an appendix here}

% \textbf{\emph{Do not put content after the references.}}
% %
% Put anything that you might normally include after the references in a separate
% supplementary file.

% We recommend that you build supplementary material in a separate document.
% If you must create one PDF and cut it up, please be careful to use a tool that
% doesn't alter the margins, and that doesn't aggressively rewrite the PDF file.
% pdftk usually works fine. 

% \textbf{Please do not use Apple's preview to cut off supplementary material.} In
% previous years it has altered margins, and created headaches at the camera-ready
% stage. 
%%%%%%%%%%%%%%%%%%%%%%%%%%%%%%%%%%%%%%%%%%%%%%%%%%%%%%%%%%%%%%%%%%%%%%%%%%%%%%%
%%%%%%%%%%%%%%%%%%%%%%%%%%%%%%%%%%%%%%%%%%%%%%%%%%%%%%%%%%%%%%%%%%%%%%%%%%%%%%%

\end{document}